\documentclass[smallextended]{svjour3_arxiv}
\usepackage{amsfonts}
\usepackage{graphicx}
\usepackage{epstopdf}
\usepackage{amsmath, amssymb, bm}
\usepackage[fixamsmath,disallowspaces]{mathtools}

\RequirePackage{fix-cm}
\smartqed  
\newtheorem{assumption}{Assumption}
\journalname{Bit Numer Math (2016) 56:995-1015, 
DOI 10.1007/s10543-015-0580-y}
\def\ODEvec{4}
\begin{document}

\title{A Note on the Motion Representation and Configuration Update in Time
Stepping Schemes for the Constrained Rigid Body}
\date{}
\author{Andreas M\"uller}
\institute{Institute of Robotics, Johannes Kepler University, Altenbergerstr. 69, 4040 Linz, Austria}
\maketitle

\begin{abstract}
The dynamics of a holonomically constrained rigid body can be modeled by
Newton-Euler equations subjected to geometric constraints. This is
frequently formulated as a differential-algebraic equation (DAE) system of
index 1. In multibody system (MBS) dynamics it is common 1) to numerically
solve this system by means of integration schemes for ordinary differential
equations (ODE), and 2) to treat the rigid body motion on the direct product
Lie group $SO\left( 3\right) \times \mathbb{R}^{3}$, although rigid body
motions form the semidirect product Lie group $SE\left( 3\right) $. It is
has been observed that the constraint satisfaction depends on which Lie
group is used as configuration space (c-space).

In this paper the problem is considered from a geometric perspective. It is
shown that the constraints are exactly satisfied by a numerical integration
scheme if they define a subgroup of the c-space. The subgroups of $SE\left(
3\right) $ have a significance for modeling mechanical systems, including
lower kinematic (Reuleaux) pairs and are implicitly used in MBS modeling. It
is concluded that $SE\left( 3\right) $ is the appropriate c-space for
numerical DAE modeling of a constrained rigid body. This result does not
immediately apply to MBS, however.

\end{abstract}

\titlerunning{Motion Representation for the Constrained Rigid Body}

\setlength{\parindent}{0mm} \parskip0.8ex

\titlerunning{ }

\authorrunning{ }

\institute{University of Michigan - Shanghai Jiao Tong University\\
	Joint Institute,
	Shanghai, China,
	andreas.mueller@ieee.org\vspace{-4ex}}

\section{Introduction}

Technical systems are modeled as constrained or restrained systems of rigid
or flexible bodies, i.e. as a multibody system (MBS). An approach frequently
used is the so-called 'absolute coordinate formulation' where the
Newton-Euler equations of individual bodies are formulated and subjected to
constraint and interaction forces \cite{Angeles2003,Schiehlen1997,Shabana}.
The corresponding numerical model is a system of differential-algebraic
equations (DAE) of index 3 that is commonly transformed to an index 1 DAE
and solved numerical using ODE integration methods. This inevitably leads to
constraint violation that are of the order of the accuracy of the
integration scheme. Various constraint stabilization methods were introduced
to remedy this problem \cite{Blajer2011,KhanAnderson2013}.

In MBS dynamics \cite{Angeles2003,Schiehlen1997,Shabana} the configuration
space (c-space) of the rigid body is historically treated as the direct
product Lie group $SO\left( 3\right) \times \mathbb{R}^{3}$. This may not be
stated explicitly but is implied by the decoupled parameterization of
rotations and translations. Rigid body motions are screw motions, however.
Hence the proper rigid body c-space is the semidirect product $SE\left(
3\right) =SO\left( 3\right) \ltimes \mathbb{R}^{3}$ --the group of isometric
and orientation preserving transformations of Euclidean space, called the
special Euclidean group in three dimensions. Moreover, only this c-space
ensures frame invariance of the motion equations \cite{Bottasso2002}.

There is a recent awareness that the use of an appropriate geometric model
of rigid body motions is crucial for the numerical performance, mainly
motivated by the development of Lie group integration schemes for MBS \cite%
{BruelsCardona2010}. An overview of the recent developments of Lie group
integration schemes can be found in \cite{Celledoni2014}. It has also been
observed that the constraint satisfaction varies dramatically for either
choice of c-space \cite{CAM2013}. Moreover, it is crucial to observe that
this applies to any integration scheme, i.e. equally to the (vector space)
integration schemes commonly used in MBS dynamics. This issue has not yet
been addressed systematically. In \cite{CAM2013} the consequences of using
the direct and semidirect product group has been analyzed for MBS by means
of numerical simulations. It is important to notice that this

In this note the case of a single constrained rigid body is considered,
which allows for a concrete statement. It will be shown that constraints are
exactly satisfied, if they constrain the body's motion to a c-space
subgroup. Noting that only subgroups of $SE\left( 3\right) $ are relevant
for MBS modeling, including lower kinematic (Reuleaux) pairs, $SE\left(
3\right) $ is deemed the appropriate c-space for numerical analysis.

\section{Problem Statement}

The motion of a rigid body can be represented as a curve $C\left( t\right) $
in a 6-dimensional Lie group $G$ --the body's c-space. Its velocity can be
defined as left-invariant vector field $V=C^{-1}\dot{C}\in \mathfrak{g}$.
The latter can be represented by a vector $\mathbf{V}\in {\mathbb{R}}^{6}$
comprising the angular and translational velocity (see section 3). The
momentum vector $\mathbf{\Pi }=\mathbf{MV}\in \mathfrak{g}^{\ast }$ is then
defined with the body-fixed inertia matrix $\mathbf{M}$. The velocity as
well as the corresponding Hamiltonian is left-invariant. The dynamics of a
free rigid body is thus governed by the \emph{Euler-Poincar\'{e} equations} 
\cite{MarsdenBook}%
\begin{equation}
\dot{\mathbf{\Pi }}-\mathbf{ad}_{\mathbf{V}}^{T}\mathbf{\Pi }=\mathbf{0}
\label{Poincare1}
\end{equation}%
on $\mathfrak{g}^{\ast }$. These are the Newton-Euler equations according to
the geometric model of rigid body motions encoded in $G$. The solution of (%
\ref{Poincare1}) is the velocity respectively momentum evolution. The actual
motion is the solution of the (left) \emph{Poisson equations} $\dot{C}=CV$,
that will be referred to as the \emph{kinematic reconstruction equations} 
\cite{MarsdenBook}.

Let the body be subjected to $m$ geometric constraints $h\left( C\right) =%
\mathbf{0}$, with constraint mapping $h:G\rightarrow \mathbb{R}^{m}$. The
corresponding velocity constraints can be expressed as $\mathbf{J}\left(
C\right) \mathbf{V}=\mathbf{0}$, where $\mathbf{J}$ is the left-trivialized
tangent mapping of $h$ referred to as constraint Jacobian. The equations of
motion (EOM) of the constrained rigid body attain the form%
\begin{align}
\dot{\mathbf{\Pi }}-\mathbf{ad}_{\mathbf{V}}^{T}\mathbf{\Pi }+\mathbf{J}^{T}%
\mathbf{\lambda }& =\mathbf{W}\left( C\right) \addtocounter{equation}{1} 
\tag{{\theequation}a}  \label{NE1a} \\
\dot{C}& =CV\   \tag{{\theequation}b}  \label{NE1b} \\
h\left( C\right) & =\mathbf{0}  \tag{{\theequation}c}  \label{NE1c}
\end{align}%
with applied forces $\mathbf{W}\in \mathfrak{g}^{\ast }$ and Lagrange
multipliers $\mathbf{\lambda }\in \mathfrak{g}^{\ast }$. This is an index 3
DAE system on the state space $G\times \mathfrak{g}$ (using $C$ and $\mathbf{%
V}$ as variables).

In many practically significant situations the constraints define a
subgroup, i.e. $H=h^{-1}\left( \mathbf{0}\right) $ is a subgroup of $G$. In
this case $H$ can be used as c-space and the dynamics is governed by the
unconstrained equations%
\begin{align}
\dot{\mathbf{\Pi }}-\mathbf{ad}_{\mathbf{V}}^{T}\mathbf{\Pi }& =\mathbf{W}%
\left( C\right) \addtocounter{equation}{1}  \tag{{\theequation}a} \\
\dot{C}& =CV\   \tag{{\theequation}b}
\end{align}%
on the state space $H\times \mathfrak{h}$. A prominent example is the heavy
top where $H=SO\left( 3\right) $. This has been the standard example for Lie
group integration schemes applied to the rigid body \cite%
{muntekaas1998,muntekaas1999,OwrenMarthinsen1999}.

The solution of (\ref{NE1b}) has the form $C\left( t\right) =C_{0}\exp 
\mathbf{X}\left( t\right) $, where $\mathbf{X}\left( t\right) $ is a curve
in $\mathfrak{g}$, and $C_{0}\in G$ the initial value. Thus the kinematic
reconstruction equations (\ref{NE1b}) on $G$ can be replaced by the system $%
\mathbf{V}=\mathrm{dexp}_{-\mathbf{X}}(\dot{\mathbf{X}})$ on $\mathfrak{g}$.
This is the basic idea behind the Munthe-Kaas (MK) method \cite%
{Hairer2006,Iserles2000,muntekaas1998,muntekaas1999}, for instance.

It is a common approach in MBS dynamics \cite{Angeles2003,Shabana} to
transform the system (\ref{NE1a}-\ref{NE1c}) to an index 1 DAE (see (\ref%
{index1}) below) that gives rise to an ODE of the form%
\begin{align}
\dot{\mathbf{V}}& =F\left( \mathbf{X},\mathbf{V}\right) %
\addtocounter{equation}{1}  \tag{{\theequation}a}  \label{ODEvec1a} \\
\dot{\mathbf{X}}& =\mathrm{dexp}_{-\mathbf{X}}^{-1}(\mathbf{V}). 
\tag{{\theequation}b}  \label{ODEvec1b}
\end{align}%
This is an ODE on the vector space $\mathbb{R}^{6}\times \mathbb{R}^{6}$
(identifying $\mathfrak{g}$ with $\mathbb{R}^{6}$), which can be solved with
any numerical integration method. Moreover it splits into the system (\ref%
{ODEvec1a}) on $\mathbb{R}^{6}$ and the system (\ref{ODEvec1b}) on the Lie
algebra $\mathfrak{g}$.

So far no particular c-space Lie group $G$ is specified. The c-space Lie
group, i.e. the geometric model of rigid body motions, determines 1) the
explicit form of the Euler-Poincar\'{e} equations, and 2) the dexp mapping,
i.e. the kinematic reconstruction equations. Moreover, (\ref{ODEvec1b}) is a
system on the Lie algebra $\mathfrak{g}$, and thus specific to the c-space.
The following question is addressed in the remainder of this note:

\begin{description}
\item[\textbf{Problem}:] How does the choice of c-space Lie group affect the
preservation of the constraint manifold $h^{-1}\left( \mathbf{0}\right) $
when numerically integrating the rigid body model (\ref{ODEvec1a},\ref%
{ODEvec1b})?
\end{description}

\section{Two Choices of Configuration Space}

The configuration of a rigid body is given by $C=\left( \mathbf{R},\mathbf{r}%
\right) $, where $\mathbf{R}\in SO\left( 3\right) $ represents the
orientation of a body-fixed reference frame (RFR) w.r.t. a spatial inertial
frame (IFR), and $\mathbf{r}\in \mathbb{R}^{3}$ is the position vector
expressed in this frame. Two choices for the c-space are used in the
literature that will be considered: 1) the direct product Lie group $%
SO\left( 3\right) \times \mathbb{R}^{3}$, and 2) the semidirect product $%
SE\left( 3\right) =SO\left( 3\right) \ltimes \mathbb{R}^{3}$ --the special
Euclidean group in three dimensions \cite{Selig}.

In the following $\widehat{\mathbf{\xi }}\in so\left( 3\right) $ denotes the
skew symmetric matrix associated to the vector $\mathbf{\xi }\in \mathbb{R}%
^{3}$. With slight abuse of notation the vector will be referred to as an
element of $so\left( 3\right) $. The right-translated differential $\mathrm{%
dexp}:\mathfrak{g}\times \mathfrak{g}\rightarrow \mathfrak{g}$ of the exp
mapping on $G$ is defined as $\mathrm{dexp}_{\mathbf{X}}(\dot{\mathbf{X}})=%
\dot{C}C^{-1}$, with $C=\exp \mathbf{X}$, which implies $\mathrm{dexp}_{-%
\mathbf{X}}(\dot{\mathbf{X}})=C^{-1}\dot{C}$.

\subsection{Direct Product Lie Group $SO\left( 3\right) \times \mathbb{R}%
^{3} $}

The direct product group is traditionally used as c-space in MBS dynamics 
\cite{Shabana}. This stems from the notion that rotations and translations
are independent, and is also reflected by the decoupled parameterization of $%
C$ in terms of rotation parameters (angles, angle-axis) and position vector.
The group multiplication is 
\begin{equation}
C_{1}\cdot C_{2}=(\mathbf{R}_{1}\mathbf{R}_{2},\mathbf{r}_{1}+\mathbf{r}_{2})
\label{MultDirectProduct}
\end{equation}%
and the inverse element is $C^{-1}=(\mathbf{R}^{T},-\mathbf{r})$. The group
is generated by its Lie algebra $so\left( 3\right) \times \mathbb{R}^{3}$,
with elements $\mathbf{X}=\left( \mathbf{\xi },\mathbf{r}\right) $, via the
exponential mapping%
\begin{equation}
\mathbf{X}=\left( \mathbf{\xi },\mathbf{r}\right) \longmapsto \exp \mathbf{X}%
=(\exp \widehat{\mathbf{\xi }},\mathbf{r}),  \label{expSO3xR3}
\end{equation}%
where $\exp \widehat{\mathbf{\xi }}$ is the exponential on $SO\left(
3\right) $ \cite{Selig}. The configuration of a rigid body is thus
parameterized by the scaled rotation axis $\mathbf{\xi }$ and the
displacement vector $\mathbf{r}$. The Lie bracket on $so\left( 3\right)
\times \mathbb{R}^{3}$ is%
\begin{equation}
\left[ \mathbf{X}_{1},\mathbf{X}_{2}\right] =\left( \mathbf{\xi }_{1}\times 
\mathbf{\xi }_{2},\mathbf{0}\right) .  \label{bracketSO3xR3}
\end{equation}%
This can be expressed as $\mathbf{ad}_{\mathbf{X}_{1}}\mathbf{X}_{2}$ with
matrix%
\begin{equation}
\mathbf{ad}_{\mathbf{X}}=\left( 
\begin{array}{cc}
\widehat{\mathbf{\xi }} & \ \ \mathbf{0} \\ 
\mathbf{0} & \ \ \mathbf{0}%
\end{array}%
\right) .  \label{adSO3R3}
\end{equation}%
The rigid body velocity $V^{\text{m}}:=C^{-1}\dot{C}\in so\left( 3\right)
\times \mathbb{R}^{3}$, that reads in vector form $\mathbf{V}^{\text{m}}=(%
\mathbf{\omega }^{\text{b}},\dot{\mathbf{r}})$, is referred to as the \emph{%
mixed velocity} since it consists of the body-fixed angular velocity $%
\mathbf{\omega }^{\text{b}}$ and the translational velocity $\dot{\mathbf{r}}
$ expressed in the spatial IFR. This is commonly used in MBS dynamics \cite%
{Shabana}.

The right-translated differential of the exp mapping is in matrix form%
\begin{equation}
\mathbf{dexp}_{\mathbf{X}}\dot{\mathbf{X}}=\left( 
\begin{array}{cc}
\mathbf{dexp}_{\mathbf{\xi }} & \ \ \mathbf{0}%
\vspace{2mm}
\\ 
\mathbf{0} & \ \ \mathbf{I}%
\end{array}%
\right)  \label{dexpSO3xR3}
\end{equation}%
so that $\mathbf{V}^{\text{m}}=\mathbf{dexp}_{-\mathbf{X}}\dot{\mathbf{X}}$
for $C=\exp \mathbf{X}\in SO\left( 3\right) \times \mathbb{R}^{3}$. Therein%
\begin{equation}
\mathbf{dexp}_{\mathbf{\xi }}=\frac{1}{\left\Vert \mathbf{\xi }\right\Vert
^{2}}%
\big%
[(\mathbf{I}-\exp \widehat{\mathbf{\xi }})\widehat{\mathbf{\xi }}+\mathbf{%
\xi \xi }^{T}%
\big%
]  \label{dexpSO3}
\end{equation}%
is the matrix form of the dexp mapping on $SO\left( 3\right) $ \cite{Selig}.
The inverse, required in (\ref{ODEvec1b}), is%
\begin{equation}
\mathbf{dexp}_{\mathbf{\xi }}^{-1}=\mathbf{I}-\frac{1}{2}\widehat{\mathbf{%
\xi }}+\left( 1-\frac{\left\Vert \mathbf{\xi }\right\Vert }{2}\cot \frac{%
\left\Vert \mathbf{\xi }\right\Vert }{2}\right) \frac{\widehat{\mathbf{\xi }}%
^{2}}{\left\Vert \mathbf{\xi }\right\Vert ^{2}}.  \label{dexpInvSO3}
\end{equation}

In order to decouple the Newton and Euler equations, the RFR is located at
the center of mass (COM). With the body-fixed inertia matrix $\mathbf{M}%
=\left( 
\begin{array}{cc}
\mathbf{\Theta }_{0} & \mathbf{0} \\ 
\mathbf{0} & m\mathbf{I}%
\end{array}%
\right) $, where $\mathbf{\Theta }_{0}$ is the inertia tensor w.r.t. the
COM, the \emph{mixed momentum} vector is $\mathbf{\Pi }^{\text{m}}=\mathbf{MV%
}^{\text{m}}=\left( \mathbf{\Theta }_{0}\mathbf{\omega }^{\text{b}},m\dot{%
\mathbf{r}}\right) \in so^{\ast }\left( 3\right) \times {\mathbb{R}}^{3}$.
The corresponding Euler-Poincar\'{e} equations (\ref{Poincare1}) are, with (%
\ref{adSO3R3}), explicitly%
\begin{eqnarray}
\mathbf{\Theta }_{0}\dot{\mathbf{\omega }}^{\text{b}}+\mathbf{\omega }^{%
\text{b}}\times \mathbf{\Theta }_{0}\mathbf{\omega }^{\text{b}} &=&\mathbf{0}
\label{NESO3xR3} \\
m\ddot{\mathbf{r}} &=&\mathbf{0.}  \notag
\end{eqnarray}%
These are indeed the decoupled Newton-Euler equations w.r.t. the COM. This
decoupling of rotations and translations prevails as apparent from (\ref%
{MultDirectProduct}),(\ref{adSO3R3}), and (\ref{dexpSO3xR3}).

The mixed velocity and momentum are left-invariant w.r.t. actions of $%
SO\left( 3\right) \times \mathbb{R}^{3}$. The latter are not proper frame
transformations. Moreover, the equations (\ref{NESO3xR3}) cannot be
transformed to another RFR by means of actions $SO\left( 3\right) \times {%
\mathbb{R}}^{3}$. This already indicates that the latter is not a proper
rigid body c-space.

\subsection{Special Euclidean Group $SE\left( 3\right) $}

Although being independent, rigid body rotations and translations are not
decoupled. Moreover a rigid body motion is a screw motion belonging to $%
SE\left( 3\right) $, and this is hence the proper rigid body c-space. This
is reflected in the group multiplication $C_{2}\cdot C_{1}=\left( \mathbf{R}%
_{2}\mathbf{R}_{1},\mathbf{r}_{2}+\mathbf{R}_{2}\mathbf{r}_{1}\right) $ that
describes frame transformations. The inverse element is $C^{-1}=\left( 
\mathbf{R},-\mathbf{R}^{T}\mathbf{r}\right) $. The Lie algebra $se\left(
3\right) =so\left( 3\right) \ltimes \mathbb{R}^{3}$ consists of elements of
the form $\mathbf{X}=\left( \mathbf{\xi },\mathbf{\eta }\right) $, and is
equipped with the Lie bracket (screw product) $\left[ \mathbf{X}_{1},\mathbf{%
X}_{2}\right] =\left( \mathbf{\xi }_{1}\times \mathbf{\xi }_{2},\mathbf{\xi }%
_{1}\times \mathbf{\eta }_{2}-\mathbf{\xi }_{2}\times \mathbf{\eta }%
_{1}\right) $ \cite{Selig}. The latter can be expressed as $\left[ \mathbf{X}%
_{1},\mathbf{X}_{2}\right] =\mathbf{ad}_{\mathbf{X}_{1}}\mathbf{X}_{2}$,
with the matrix%
\begin{equation}
\mathbf{ad}_{\mathbf{X}}=\left( 
\begin{array}{cc}
\widehat{\mathbf{\xi }} & \ \ \mathbf{0} \\ 
\widehat{\mathbf{\eta }} & \ \ \widehat{\mathbf{\xi }}%
\end{array}%
\right) .  \label{adse3}
\end{equation}%
The exponential mapping attains, with (\ref{dexpSO3}), the explicit form 
\begin{equation}
\mathbf{X}=\left( \mathbf{\xi },\mathbf{\eta }\right) \longmapsto \exp 
\widehat{\mathbf{X}}{}=(\exp \widehat{\mathbf{\xi }},\mathbf{dexp}_{\mathbf{%
\xi }}\mathbf{\eta }).  \label{expX}
\end{equation}%
Therewith the rigid body configuration is parameterized in therms of \emph{%
screw coordinates} $\mathbf{X}\in se\left( 3\right) $.

The velocity $V^{\text{b}}=C^{-1}\dot{C}\in se\left( 3\right) $ is called
the \emph{body-fixed velocity} screw (also called \emph{body-twist}). In
vector form $\mathbf{V}^{\text{b}}=(\mathbf{\omega }^{\text{b}},\mathbf{v}^{%
\text{b}})$, where $\mathbf{v}^{\text{b}}=\mathbf{R}^{T}\dot{\mathbf{r}}$ is
the body-fixed translational velocity.

The right-translated differential of exp possesses different representations
in matrix form as e.g. reported in \cite{Bottasso1998,ParkChung2005}, so
that $\mathbf{V}^{\text{b}}=\mathbf{dexp}_{-\mathbf{X}}\dot{\mathbf{X}}$ for 
$C=\exp \mathbf{X}\in SE\left( 3\right) $. The inverse of dexp in matrix
form has been derived in \cite{Selig}. A computationally efficient form is
reported in \cite{Bottasso1998,ParkChung2005}%
\begin{equation}
\mathbf{dexp}_{\mathbf{X}}^{-1}=\left( 
\begin{array}{cc}
\mathbf{dexp}_{\mathbf{\xi }}^{-1} & \mathbf{0}%
\vspace{2mm}
\\ 
\mathbf{U} & \mathbf{dexp}_{\mathbf{\xi }}^{-1}%
\end{array}%
\right)  \label{dexpInvSE3Park}
\end{equation}%
with%
\begin{equation}
\mathbf{U}\left( \mathbf{X}\right) =\frac{1-\gamma }{\left\Vert \mathbf{\xi }%
\right\Vert ^{2}}\left( \widehat{\mathbf{\eta }}\widehat{\mathbf{\xi }}+%
\widehat{\mathbf{\xi }}\widehat{\mathbf{\eta }}\right) +\frac{p}{\left\Vert 
\mathbf{\xi }\right\Vert ^{3}}\left( \frac{1}{\beta }+\gamma -2\right) 
\widehat{\mathbf{\xi }}^{2}-\frac{1}{2}\widehat{\mathbf{\eta }}
\end{equation}%
and $\gamma :=\frac{2}{\left\Vert \mathbf{\xi }\right\Vert }\cot \frac{%
\left\Vert \mathbf{\xi }\right\Vert }{2}$, $\alpha :=\frac{2}{\left\Vert 
\mathbf{\xi }\right\Vert }\sin \frac{\left\Vert \mathbf{\xi }\right\Vert }{2}%
\cos \frac{\left\Vert \mathbf{\xi }\right\Vert }{2},\beta :=\frac{4}{%
\left\Vert \mathbf{\xi }\right\Vert ^{2}}\sin ^{2}\frac{\left\Vert \mathbf{%
\xi }\right\Vert }{2}$. Here $p:=\mathbf{\xi }\cdot \mathbf{\eta }%
/\left\Vert \mathbf{\xi }\right\Vert ^{2}$ is the pitch of the screw.

With the left-invariant body-fixed momentum screw $\mathbf{\Pi }^{\text{b}}=%
\mathbf{MV}^{\text{b}}=(\mathbf{\Theta }_{0}\mathbf{\omega }^{\text{b}},m%
\mathbf{v}^{\text{b}})\in se^{\ast }\left( 3\right) $ the dynamics of the
rigid body is governed by the Euler-Poincar\'{e} equations (\ref{Poincare1})
that, with (\ref{adse3}), are the body-fixed Newton-Euler equations%
\begin{eqnarray}
\mathbf{\Theta }_{0}\dot{\mathbf{\omega }}^{\text{b}}+\mathbf{\omega }^{%
\text{b}}\times \mathbf{\Theta }_{0}\mathbf{\omega }^{\text{b}} &=&\mathbf{0}
\label{NESE3} \\
m\dot{\mathbf{v}}^{\text{b}}+m\mathbf{\omega }\times \mathbf{v}^{\text{b}}
&=&\mathbf{0.}  \notag
\end{eqnarray}

\section{Constraint Satisfaction in Numerical Time Integration of the Index
1 DAE}

\subsection{The Associated Index 1 DAE}

Consider the single rigid body with motion equations (\ref{NE1a}-\ref{NE1c}%
). Following the standard method in MBS dynamics, this is reformulated as
index 1 DAE \cite{Angeles2003,Schiehlen1997,Shabana}%
\begin{equation}
\left( 
\begin{array}{cc}
\mathbf{M} & \mathbf{J}^{T} \\ 
\mathbf{J} & \mathbf{0}%
\end{array}%
\right) \left( 
\begin{array}{c}
\dot{\mathbf{V}} \\ 
\mathbf{\lambda }%
\end{array}%
\right) =\left( 
\begin{array}{c}
\mathbf{ad}_{\mathbf{V}}^{T}\mathbf{\Pi }+\mathbf{W} \\ 
\mathbf{\eta }%
\end{array}%
\right)  \label{index1}
\end{equation}%
by complementing (\ref{NE1a}) with the acceleration constraints, $\mathbf{J}%
\left( C\right) \dot{\mathbf{V}}=\mathbf{\eta }\left( C,\mathbf{V}\right) $.
Then the system (\ref{index1}) is solved for $\dot{\mathbf{V}}$ (and $%
\mathbf{\lambda }$ as by-product), which yields an ODE of the form (\ref%
{ODEvec1a}). Then the ODE (\ODEvec) is solved with a numerical (vector
space) integration scheme.

\subsection{Constraint Satisfaction for a general C-Space Lie Group}

The constraints are satisfied during the integration as long as the
configuration update step of the applied numerical integration scheme
returns a configuration in $h^{-1}\left( \mathbf{0}\right) $. It is
well-known, however, that numerically solving (\ODEvec) with ODE integration
schemes generally leads to violations of the (now hidden) geometric
constraints (\ref{NE1c}). To reduce this numerical drift, constraint
stabilization methods have been proposed that either amend the motion
equations \cite{Baumgarte,KhanAnderson2013} or correct the numerical
solution by projecting it to the constraint manifold $h^{-1}\left( \mathbf{0}%
\right) $ \cite{Ascher1995,Blajer2011,Brauchli1991,TerzeNaudet2008}. All
these methods aim at minimizing or correcting, rather than avoiding,
constraint violations.

The subsequent analysis of the constraint satisfaction makes use the
parameterization of the rigid body configuration in terms of $\mathbf{X}\in 
\mathfrak{g}$. Denoting with $\mathbf{X}^{\left( i\right) }:=\mathbf{X}%
\left( t_{i}\right) $, the configuration at time step $t_{i}$ is $C\left(
t_{i}\right) =\exp \mathbf{X}^{\left( i\right) }$, with $\mathbf{X}^{\left(
i\right) }=\mathbf{X}^{\left( i-1\right) }+$ $\mathbf{\Phi }^{\left(
i\right) }$. The increment $\mathbf{\Phi }^{\left( i\right) }\in \mathfrak{g}
$ is found by numerically solving (\ODEvec). To this end any explicit or
implicit vector space integration scheme can be used.

A (possibly implicit) numerical integration scheme determines $\mathbf{\Phi }%
^{\left( i\right) }$ by (usually) linear combination of the right hand side
of (\ref{ODEvec1b}) evaluated at $s$ intermediate time steps. The right hand
side is given by the inverse of dexp, which possesses the series expansion%
\begin{equation}
\mathrm{dexp}_{\mathbf{X}}^{-1}\left( \mathbf{Y}\right) =\sum_{i\geq 0}\frac{%
B_{i}}{i!}\mathrm{ad}_{\mathbf{X}}^{i}\left( \mathbf{Y}\right)
\label{dexpInvSeries}
\end{equation}%
with the Bernoulli numbers $B_{i}$. Hence the following assumption is
feasible for the general class of integration schemes that construct the
solution as linear combination of the right hand side of (\ref{ODEvec1b}).

\begin{assumption}
The configuration update increment is determined as%
\begin{equation}
\mathbf{\Phi }^{\left( i\right) }=\sum_{j=1}^{s}\sum_{r\geq 0}\alpha _{ij}%
\mathbf{ad}_{\mathbf{V}^{\left( j\right) }}^{r}\mathbf{V}^{\left( i\right) }
\label{sum}
\end{equation}%
denoting $\mathbf{V}^{\left( j\right) }:=\mathbf{V}(t_{i-1}+c_{j}\Delta t,%
\mathbf{q}_{j})$ with some real coefficients $\alpha _{ij}$ and $c_{j}$
specific to the integration scheme.
\end{assumption}

Also the numerically obtained solution $\mathbf{V}\left( t\right) $ of (%
\ODEvec) does not necessarily satisfy the velocity constraints $\mathbf{J}%
\left( C\right) \mathbf{V}=\mathbf{0}$. This linear condition can be easily
satisfied, however.

\begin{assumption}
It is assumed that the velocity $\mathbf{V}^{\left( j\right) }$ satisfies
the kinematic constraints at all time steps $t_{j}=t_{i-1}+c_{j}\Delta t$.
\end{assumption}

This gives rise to the following.

\begin{theorem}
\label{corollary1}The kinematic motion constraints of a rigid body are
satisfied by a configuration update step in terms of linear combinations of
velocity samples $\mathbf{V}^{\left( j\right) }$, if the constraints
restrict the body's motion to a subgroup of its c-space Lie group.
\end{theorem}

\begin{proof}
With assumption 2, the body velocities at the intermediate configurations
belong to the subspace of $\mathfrak{g}$ defined by the velocity
constraints. The $r$-fold nested Lie brackets $\mathbf{ad}_{\mathbf{V}%
^{\left( j\right) }}^{r}\mathbf{V}^{\left( i\right) }$, $r\geq 0$, in the
construction (\ref{sum}) of $\mathbf{\Phi }^{\left( i\right) }$ form a basis
for the smallest Lie subalgebra of $\mathfrak{g}$ comprising all $\mathbf{V}%
^{\left( i\right) }$ at intermediate time steps $t_{i-1}+c_{j}\Delta t$. If
the constraints restrict the motion to a subgroup $H=h^{-1}\left( \mathbf{0}%
\right) \subset G$, and thus the velocities to the corresponding subalgebra
of $\mathfrak{g}$, the brackets in terms of feasible velocities in (\ref{sum}%
) form a basis for this subalgebra. Consequently, $\mathbf{\Phi }^{\left(
i\right) }$ belongs to the smallest Lie subalgebra of $\mathfrak{g}$
containing the constrained body velocity $\mathbf{V}$, and the configuration
increment $C\left( t_{i}\right) $ belongs to the corresponding subgroup of
the c-space Lie group $G$, (occasionally called the \textit{completion group}%
). Hence the constraints are satisfied regardless of the order and accuracy
of the integration scheme.
\end{proof}

\subsection{Constraint Satisfaction for $SE\left( 3\right) $ and $SO\left(
3\right) \times {\mathbb{R}}^{3}$}

The above theorem applies to arbitrary c-space Lie groups. Clearly the two
c-space Lie groups of interest have different subgroups, and it is important
to analyze which subgroups do actually correspond to practically relevant
constraints. This is vital since traditionally all formulations for rigid
body MBS use $SO\left( 3\right) \times \mathbb{R}^{3}$ as c-space (at least
implicitly by treating rotations and translations decoupled), whereas only a
few formulations for flexible bodies employ $SE\left( 3\right) $ such as 
\cite{BorriBottasso1994,Borri2002,Bottasso1998,Bottasso2002}. Hence a
legitimate question is whether using a proper rigid body c-space Lie group
will alleviate or even eliminate constraint violations.

The motion of a rigid body is represented by the motion of its body-fixed
RFR relative to the IFR, summarized as $C=\left( \mathbf{R},\mathbf{r}%
\right) $.

The 10 subgroups of $SE\left( 3\right) $ are listed in table \ref%
{TableSE3subgroups} (adopted from \cite{Selig}), where also the
corresponding lower kinematic pairs (Reuleaux pairs) are indicated. $%
SE\left( 3\right) $ represents rigid body motions, i.e. screw motions,
respecting the coupling of rotation and translation. This motion
representation is invariant w.r.t. the change of RFR. For instance, if the
RFR is located at the center of rotation, elements of $SO\left( 3\right) $,
as a $SE\left( 3\right) $ subgroup, have the form $C=\left( \mathbf{R},%
\mathbf{0}\right) $. This is occasionally called the standard
representation. For a general location of RFR, $SO\left( 3\right) $ as a
subgroup consists of elements $C=\left( \mathbf{R},\mathbf{r}\right) $ such
that there is a $M\in SE\left( 3\right) $ that transforms them to the
standard form via $MCM^{-1}$.

The 10 subgroups of the direct product $SO\left( 3\right) \times \mathbb{R}%
^{3}$ are listed in table \ref{TableSO3xR3subgroups}. It does not represent
frame transformations. The translations are decoupled from the rotations.
For instance, elements of the subgroup of rotations about a fixed pole have
necessarily the form $C=\left( \mathbf{R},\mathbf{0}\right) $, which is not
a frame-invariant representation. That is, the RFR origin is always the
center of rotation. The subscript '0' in $SO_{0}%
\hspace{-0.6ex}%
\left( 3\right) $ in table \ref{TableSO3xR3subgroups} signifies that pure
rotations are always about the RFR origin. Consequently, the direct product
cannot describe the screw motion of a body-fixed RFR at a generic location.
Recall that commonly in MBS dynamics it is used as c-space model even though!

MBS models of technical systems predominantly comprise lower pair joints.
Such joints constrain the motion to subgroups of $SE\left( 3\right) $ (not
of $SO\left( 3\right) \times \mathbb{R}^{3}$), corresponding to
lower-dimensional screw motions. Then, if $SE\left( 3\right) $ is used as
c-space Lie group, $H=h^{-1}\left( \mathbf{0}\right) $ is in fact a subgroup
of the c-space. $H$ is the \emph{isotropy group} of the joint, i.e. the
subgroup of motions leaving the contact surface of the joints invariant.
Each of the six types of lower pairs corresponds to such a subgroup as
indicated in table \ref{TableSE3subgroups}. Consequently the joint
constraints are satisfied by an $SE\left( 3\right) $ update step if a rigid
body is connected to the ground by a lower pair joint.

\begin{corollary}
The geometric constraints imposed on a rigid body by lower pair joints (and
general restrictions to $SE\left( 3\right) $ subgroups) are satisfied by a
configuration update step in terms of linear combinations of velocity
samples $\mathbf{V}^{\left( j\right) }$, if $SE\left( 3\right) $ is used as
c-space. Using $SO\left( 3\right) \times \mathbb{R}^{3}$ the update
satisfies constraints restricting to $SE\left( 3\right) $ subgroups that
are, as manifolds, identical to subgroups of $SO\left( 3\right) \times 
\mathbb{R}^{3}$, i.e. planar motions, Sch\"{o}nflies motions, and pure
translations.
\end{corollary}

\begin{remark}
\label{remarkPlanr}Given the velocity of a motion in a $SE\left( 3\right) $
subgroup. The $SO\left( 3\right) \times {\mathbb{R}}^{3}$ update does not
yield a configuration in that subgroup. For instance, the update step with a
velocity $\mathbf{V}=\left( \mathbf{\omega },\mathbf{v}\right) \in so\left(
3\right) \subset se\left( 3\right) $ leads to an independent rotational and
translational increment, i.e. not in $SO\left( 3\right) $.\newline
There are, however, two special cases where the update with both groups
perform equal: planar and Sch\"{o}nflies motions. This is so because in $%
SO\left( 2\right) \ltimes \mathbb{R}^{2}$ (planar motions) and $SO\left(
2\right) \ltimes \mathbb{R}^{3}$ (Sch\"{o}nflies motions), the action of $%
SO\left( 2\right) $ leaves ${\mathbb{R}}^{2}$ and ${\mathbb{R}}^{3}$
invariant, so that as manifolds they are identical to $SO_{0}%
\hspace{-0.6ex}%
\left( 2\right) \times \mathbb{R}^{2}$ and $SO_{0}%
\hspace{-0.6ex}%
\left( 2\right) \times \mathbb{R}^{3}$, respectively. Thus, for constraints
restricting to these subgroups, updates with both c-spaces perform equally
(see section \ref{secPlanarExamp}. The same applies to pure translations.
\end{remark}

\begin{remark}
Corollary \ref{corollary1} does not regard the accuracy of the integration
scheme. It rather says that the constraints are exactly satisfied regardless
of the accuracy of the integration scheme or step size.
\end{remark}

\begin{remark}[Parameterization of $SE\left( 3\right) $]
It was concluded that $SE\left( 3\right) $ shall be used as c-space. The
semidirect product requires according parameterization in terms of screw
coordinates $\mathbf{X}$ in (\ref{expX}). This is different to what is
currently used in MBS modeling. The main difference is that the translation
vector is not used as coordinate but is determined by the screw coordinates.
\end{remark}

\begin{remark}[Relevance for Lie group integration methods]
The ODE system (\ODEvec) arose from replacing $\dot{C}=CV$ by $\mathbf{dexp}%
_{-\mathbf{X}}\dot{\mathbf{X}}$, in order to analyze the standard vector
space integration schemes used in MBS dynamics. Resorting to the original
equation leads to the ODE
\end{remark}

\begin{align}
\dot{\mathbf{V}}& =F\left( \mathbf{X},\mathbf{V}\right) %
\addtocounter{equation}{1}  \tag{{\theequation}a} \\
\dot{C}& =CV.  \tag{{\theequation}b}
\end{align}%
This is an ODE on the state space Lie group $G\times \mathbb{R}^{6}$. Hence
Lie group integration schemes like the MK scheme \cite%
{muntekaas1998,muntekaas1999}, or the schemes adopting the Newmark and
generalized $\alpha $ scheme to the Lie group setting of constrained rigid
body dynamics \cite{BruelsCardona2010,BruelsCardonaArnold2012} can be
applied. Since these schemes use (local) canonical coordinates (of first
kind) to express the configuration update via the exponential mapping, the
above corollary also apply to the Lie group integration methods. In
particular, MK schemes solve the substitute equation (\ref{ODEvec1b}).

\begin{table}[h] \centering%

\begin{tabular}[b]{clll}
\hline
$n$ & \textbf{Subgroup} & \textbf{Kinematic Meaning} & \textbf{Lower
Kinematic Pair?} \\ \hline
1 & $\mathbb{R}$ & 1-dim. translation along some axis & Yes (Prismatic Joint)
\\ 
1 & $SO\left( 2\right) $ & 1-dim. rotation about arbitrary fixed axis & Yes
(Revolute Joint) \\ 
1 & $H_{p}$ & screw motion about arbitrary axis with finite pitch & Yes
(Screw Joint) \\ 
2 & $\mathbb{R}^{2}$ & 2-dim. planar translation & No \\ 
2 & $SO\left( 2\right) \ltimes \mathbb{R}$ & translation along arbitrary
axis \& rotation along this axis & Yes (Cylindrical Joint) \\ 
3 & $\mathbb{R}^{3}$ & spatial translations & No \\ 
3 & $SO\left( 3\right) $ & spatial rotations about arbitrary fixed point & 
Yes (Spherical Joint) \\ 
3 & $H_{p}\ltimes \mathbb{R}^{2}$ & translation in a plane + screw motion $%
\bot $ to this plane (pitch $p$) & No \\ 
3 & $SO\left( 2\right) \ltimes \mathbb{R}^{2}=SE\left( 2\right) $ & planar
motions & Yes (Planar Joint) \\ 
4 & $SO\left( 2\right) \ltimes \mathbb{R}^{3}=SE\left( 2\right) \ltimes 
\mathbb{R}$ & planar motions + spatial translations (Sch\"{o}nflies motion)
& No \\ 
6 & $SE\left( 3\right) $ & spatial motion & No \\ \hline
\end{tabular}
\caption{$n$-dimensional subgroups of $SE(3)$. It is indicated whether the subgroup
corresponds to a lower kinematic pair.}\label{TableSE3subgroups}%
\end{table}%

\begin{table}[h] \centering%
\begin{tabular}{clll}
\hline
$n$ & \textbf{Subgroup} & \textbf{Kinematic Meaning} & \textbf{Typical
Element} \\ \hline
1 & $\mathbb{R}$ & 1-dim. translation & $C=\left( \mathbf{I},\mathbf{r}%
\right) ,\mathbf{r}\in \mathbb{R}$ \\ 
1 & $SO_{0}%
\hspace{-0.6ex}%
\left( 2\right) $ & 1-dim. rotation about RFR origin & $C=\left( \mathbf{R},%
\mathbf{0}\right) ,\mathbf{R}\in SO\left( 2\right) $ \\ 
2 & $\mathbb{R}^{2}$ & 2-dim. planar translation & $C=\left( \mathbf{I},%
\mathbf{r}\right) ,\mathbf{r}\in \mathbb{R}^{2}$ \\ 
2 & $SO_{0}%
\hspace{-0.6ex}%
\left( 2\right) \times \mathbb{R}$ & 1-dim. rotation about RFR origin and
decoupled 1-dim. translation & $C=\left( \mathbf{R},\mathbf{r}\right) ,%
\mathbf{R}\in SO\left( 2\right) ,\mathbf{r}\in \mathbb{R}$ \\ 
3 & $\mathbb{R}^{3}$ & 3-dim. translation & $C=\left( \mathbf{I},\mathbf{r}%
\right) ,\mathbf{r}\in \mathbb{R}^{3}$ \\ 
3 & $SO_{0}%
\hspace{-0.6ex}%
\left( 3\right) $ & spatial rotation about RFR origin & $C=\left( \mathbf{R},%
\mathbf{0}\right) ,\mathbf{R}\in SO\left( 3\right) $ \\ 
3 & $SO_{0}%
\hspace{-0.6ex}%
\left( 2\right) \times \mathbb{R}^{2}$ & 1-dim. rotation about RFR origin
and decoupled 2-dim. translation & $C=\left( \mathbf{R},\mathbf{r}\right) ,%
\mathbf{R}\in SO\left( 2\right) ,\mathbf{r}\in \mathbb{R}^{2}$ \\ 
4 & $SO_{0}%
\hspace{-0.6ex}%
\left( 2\right) \times \mathbb{R}^{3}$ & 1-dim. rotation about RFR origin
and decoupled 3-dim. translation & $C=\left( \mathbf{R},\mathbf{r}\right) ,%
\mathbf{R}\in SO\left( 2\right) ,\mathbf{r}\in \mathbb{R}^{3}$ \\ 
4 & $SO_{0}%
\hspace{-0.6ex}%
\left( 3\right) \times \mathbb{R}$ & spatial rotation about RFR origin and
decoupled 1-dim. translation & $C=\left( \mathbf{R},\mathbf{r}\right) ,%
\mathbf{R}\in SO\left( 3\right) ,\mathbf{r}\in \mathbb{R}$ \\ 
5 & $SO_{0}%
\hspace{-0.6ex}%
\left( 3\right) \times \mathbb{R}^{2}$ & spatial rotation about RFR origin
and decoupled 2-dim. translation & $C=\left( \mathbf{R},\mathbf{r}\right) ,%
\mathbf{R}\in SO\left( 3\right) ,\mathbf{r}\in \mathbb{R}^{2}$ \\ 
6 & $SO_{0}%
\hspace{-0.6ex}%
\left( 3\right) \times \mathbb{R}^{3}$ & spatial rotation about RFR origin
and decoupled 3-dim. translation & $C=\left( \mathbf{R},\mathbf{r}\right) ,%
\mathbf{R}\in SO\left( 3\right) ,\mathbf{r}\in \mathbb{R}^{3}$ \\ \hline
\end{tabular}%
\caption{$n$-dimensinal subgroups of $SO(3)\times {\Bbb R}^3$. The subscript '0' in
$SO_0\hspace{-0.4ex}(3)$ indicates that the rotations are about the RFR
origin.}\label{TableSO3xR3subgroups}%
\end{table}%

\section{Examples}

A few examples are presented that illustrate the above discussion. In all
examples the rigid body is a rectangular aluminum solid with side lengths $%
0.3\times 0.15\times 0.05$\ m. The body-fixed RFR is located at the COM, as
shown in figure \ref{fig1}a). Its mass is $m=6.075$~kg, and its inertia
tensor w.r.t. the COM is $\mathbf{\Theta }_{0}=\mathrm{diag}%
~(0.0126563,0.0468281,0.0569531)\,$kg\thinspace m$^{2}$. In the following
results are presented for some lower pair joints and a higher kinematic
pair. The respective joint connects the body to the IFR at the ground. The
joint is located on the body at the point $\mathbf{p}=\left(
-0.15,0.0375,0\right) $m expressed in the body-fixed RFR.

The EOM (\ref{ODEvec1a},\ref{ODEvec1b}) are evaluated by solving (\ref%
{index1}). The Newton-Euler equations therein depend on the c-space. These
are (\ref{NESO3xR3}) for the direct product (using mixed velocities), and (%
\ref{NESE3}) for the semidirect product c-space (using body-fixed twists).
The constraints Jacobian $\mathbf{J}$ is joint specific.

The EOM (\ref{ODEvec1a},\ref{ODEvec1b}) were integrated for 10 s with an
explicit 4th order Runge-Kutta (RK4) method using the three step sizes $%
\Delta t=10^{-2}$s$,10^{-3}$s$,10^{-4}$s. For all examples the initial
configuration was $C_{0}=(\mathbf{I},-\mathbf{p})$, as in figure \ref{fig1}%
a).

\subsection{Spherical Joint --Heavy Top}

Connecting the body to the ground with a spherical joint (ball-and-socket
joint) yields the model of a heavy top. The heavy top became the standard
example for Lie group modeling and integration methods \cite%
{Bobenko,BruelsCardona2010,BruelsCardonaArnold2012,CelledoniOwren1999,EngoMarthinsen1998,EngoMarthinsen2001,Hairer2006}%
.

The spherical joint imposes the system of three position constraints%
\begin{equation}
h\left( C\right) :=\mathbf{r}+\mathbf{Rp}=\mathbf{0.}  \label{geomContsTop}
\end{equation}%
Differentiating twice w.r.t. time yields the acceleration constraints in
terms of body-fixed twists $\mathbf{V}^{\text{b}}=(\mathbf{\omega }^{\text{b}%
},\mathbf{v}^{\text{b}})$ 
\begin{equation}
\mathbf{J}\dot{\mathbf{V}}^{\text{b}}=\mathbf{R(}\widehat{\mathbf{\omega }}%
\widehat{\mathbf{\omega }}\mathbf{p}+\widehat{\mathbf{\omega }}\mathbf{v}),%
\text{ with }\mathbf{J=}\left( 
\begin{array}{cc}
\mathbf{R}\widehat{\mathbf{p}} & \ \ -\mathbf{R}%
\end{array}%
\right)  \label{accConstrTopSE3}
\end{equation}%
used in the $SE\left( 3\right) $ model, and in terms of mixed velocities $%
\mathbf{V}^{\text{m}}=(\mathbf{\omega }^{\text{b}},\dot{\mathbf{r}})$ 
\begin{equation}
\mathbf{J}\dot{\mathbf{V}}^{\text{m}}=\mathbf{R}\widehat{\mathbf{\omega }}%
\widehat{\mathbf{\omega }}\mathbf{p},\text{ with }\mathbf{J=}\left( 
\begin{array}{cc}
\mathbf{R}\widehat{\mathbf{p}} & \ \ -\mathbf{I}%
\end{array}%
\right)  \label{accConstrTopSO3xR3}
\end{equation}%
used in the $SO\left( 3\right) \times \mathbb{R}^{3}$ model. The EOM were
integrated with initial velocity $\mathbf{\omega }_{0}^{\text{b}}=(2\pi ,\pi
,0.2\pi )$rad/s.

The geometric constraints define the 3-dimensional constraint manifold $%
H=h^{-1}\left( \mathbf{0}\right) $. This is the manifold of spherical
displacements around the IFR origin. In $SE\left( 3\right) $ this is the
subgroup $H=SO\left( 3\right) \subset SE\left( 3\right) $. But $H$ is not a
subgroup of $SO\left( 3\right) \times \mathbb{R}^{3}$. It can only be a
subgroup of the latter, if the body-fixed RFR is located at the rotation
center. According to corollary \ref{corollary1} the $SE\left( 3\right) $
update should lead to exact constraint satisfaction while $SO\left( 3\right)
\times \mathbb{R}^{3}$ could not. This is confirmed by the position error $%
\varepsilon :=\left\Vert h\right\Vert $ shown in figure \ref%
{figErrorSpherical}. The constraint satisfaction of the direct product
formulation is dictated by the 4th order accuracy of the integration scheme. 
\begin{figure}[t]
\vspace{-1ex} 
\centerline{a)~\includegraphics[width=8.5cm]{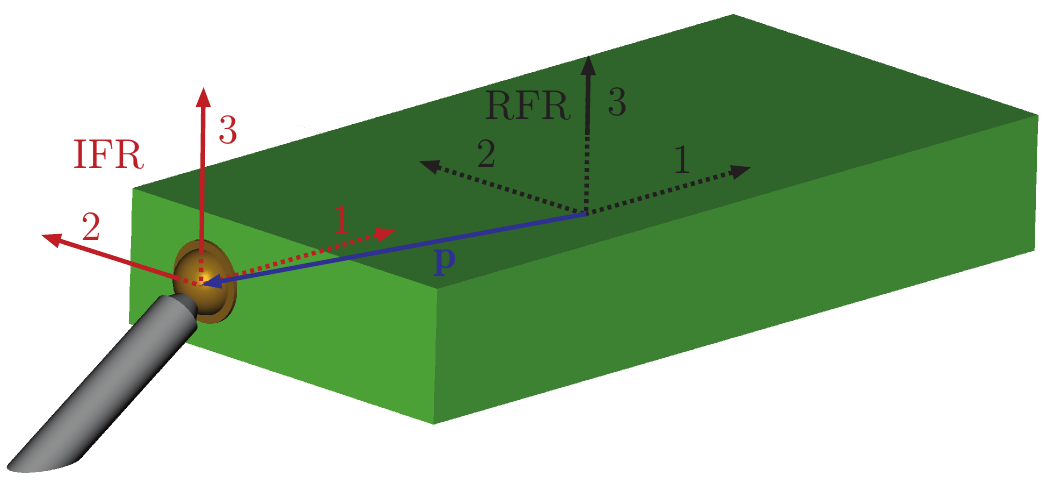}~~~~~
b)~\includegraphics[width=4cm]{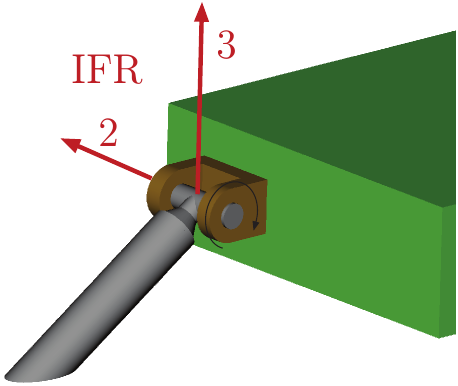}
} \vspace{-1.5ex}
\caption{Assigment of IFR and RFR, and kinematics of the a) spehrical joint,
and b) revolute joint.}
\label{fig1}
\end{figure}
\begin{figure}[h]
\vspace{-1ex} 
\centerline{a)~\includegraphics[width=8.5cm]{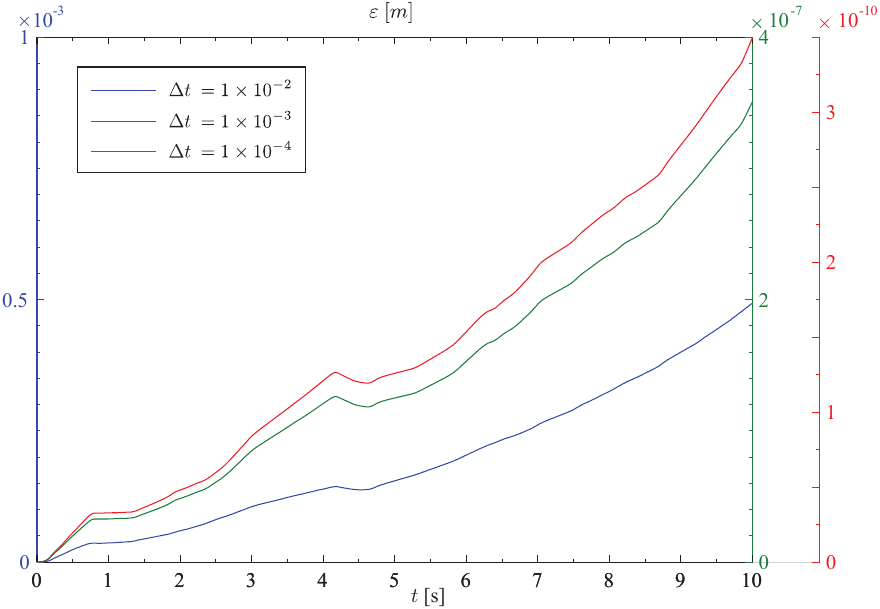}~~~~
b)~\includegraphics[width=8.3cm]{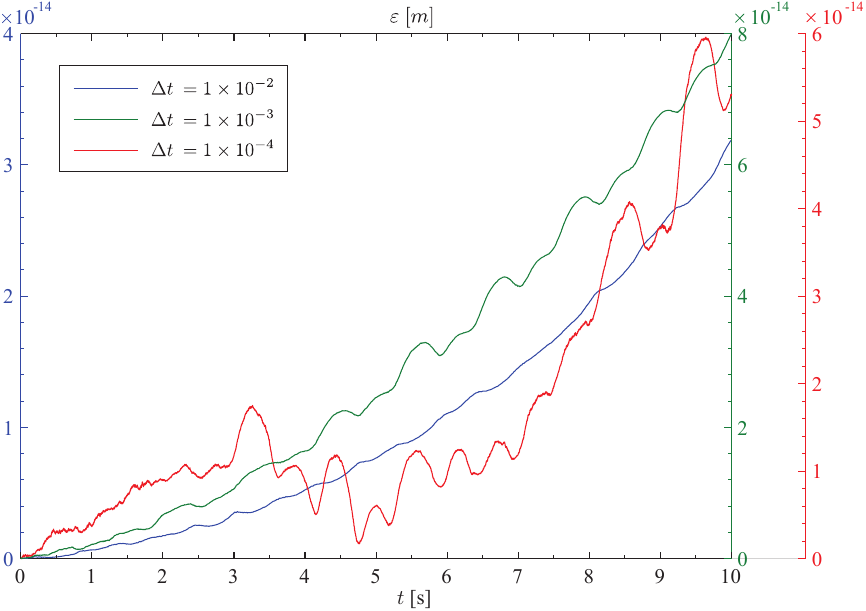}
} \vspace{-1.5ex}
\caption{Error in the spherical joint position constraints for a) $%
SO(3)\times {\mathbb{R}^{3}}$, and b) $SE(3)$.}
\label{figErrorSpherical}
\end{figure}

\subsection{Revolute Joint}

With the orientation of RFR shown in figure \ref{fig1}b), the joint axis is
aligned with the 2-axis, and the revolute joint imposes the five geometric
constraints%
\begin{eqnarray}
\mathbf{r}+\mathbf{Rp} &=&\mathbf{0}  \label{RevConstr1} \\
R_{12} &=&0  \label{RevConstr2} \\
R_{23} &=&0  \label{RevConstr3}
\end{eqnarray}%
where $\mathbf{R}=(R_{ij})$. The system of acceleration constraints in terms
of body-fixed twists consists of (\ref{accConstrTopSE3}) and the two
constraints 
\begin{eqnarray}
\ddot{\omega}_{1}^{\text{b}} &=&0  \label{RevConstr4} \\
\ddot{\omega}_{3}^{\text{b}} &=&0  \label{RevConstr5}
\end{eqnarray}%
where $\mathbf{\omega }^{\text{b}}=(\omega _{1}^{\text{b}},\omega _{2}^{%
\text{b}},\omega _{3}^{\text{b}})$. In terms of the mixed velocities the
acceleration constraints consists of (\ref{accConstrTopSO3xR3}) and the two
constraints (\ref{RevConstr4}),(\ref{RevConstr5}). The EOM were integrated
with initial velocity $\mathbf{\omega }_{0}^{\text{b}}=(0,2\pi ,0)$ rad/s.

The revolute joint constraints define the subgroup $H=SO\left( 2\right)
\subset SE\left( 3\right) $ of rotations about a fixed axis. Since the axis
does not pass through the origin of the RFR it is not a $SO\left( 3\right)
\times \mathbb{R}^{3}$ subgroup. The consequences are apparent in figure \ref%
{figErrorRevJoint}. Both exactly satisfy the rotation constraints, since a
solution of (\ref{index1}) only has non-zero components for rotations about
the rotation axis. 
\begin{figure}[h]
\vspace{-1ex} 
\centerline{a)~\includegraphics[width=8.3cm]{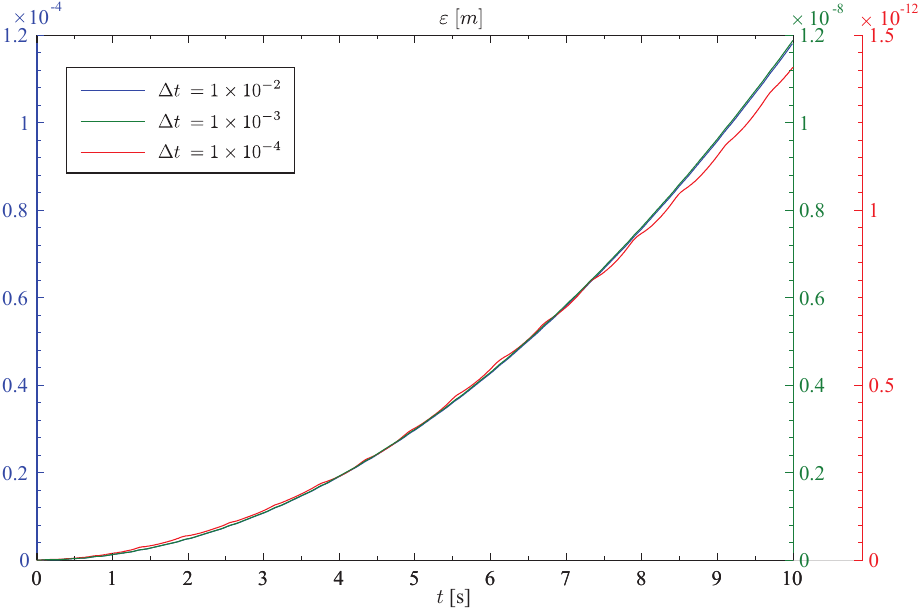}~~~~
b)~\includegraphics[width=8.3cm]{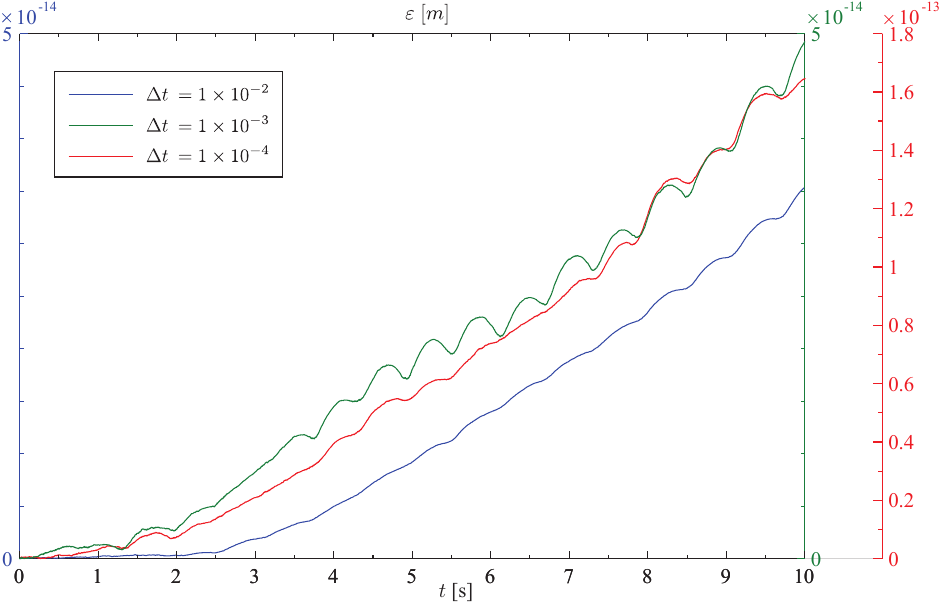}
}\vspace{-1.5ex}
\caption{Error in the revolute joint position constraints for a) $%
SO(3)\times {\mathbb{R}^{3}}$, and b) $SE(3)$.}
\label{figErrorRevJoint}
\end{figure}

\subsection{Cylindrical Joint}

A cylindrical joint allows for rotation about and translation along a fixed
axis. The RFR is allocated as for the revolute joint in figure \ref{fig1}a).
The according system of four geometric constraints consists of the first and
third equation of (\ref{RevConstr1}) and the two equations (\ref{RevConstr2}%
),(\ref{RevConstr3}). The system of acceleration constraints in terms of
body-fixed twists consists of the first and third equation of (\ref%
{accConstrTopSE3}) and the two constraints (\ref{RevConstr4}),(\ref%
{RevConstr5}). Using mixed velocities these are the first and third equation
of (\ref{accConstrTopSO3xR3}) and (\ref{RevConstr4}),(\ref{RevConstr5}).

The EOM were integrated with initial velocity $\mathbf{\omega }_{0}^{\text{b}%
}=(0,2\pi ,0)$ rad/s and $\mathbf{v}_{0}^{\text{b}}=\dot{\mathbf{r}}=\mathbf{%
p}\times \mathbf{\omega }_{0}^{\text{b}}+(0,1,0)$ m/s. The numerical results
are the same as in figure \ref{figErrorRevJoint} for the revolute joint.
This is clear since the constraints define the direct product subgroup $%
SO\left( 2\right) \times \mathbb{R}$ of $SE\left( 3\right) $.

\begin{figure}[t]
\vspace{-1ex} 
\centerline{a)~\includegraphics[width=4.7cm]{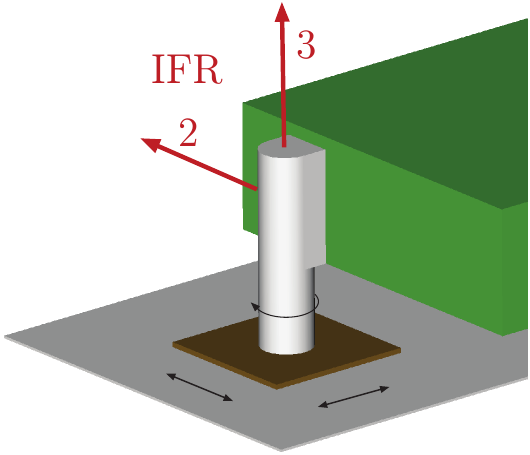}~~~~~
b)~\includegraphics[width=4cm]{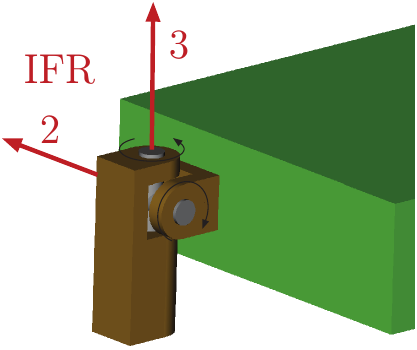}~~~~~
c)~\includegraphics[width=4.7cm]{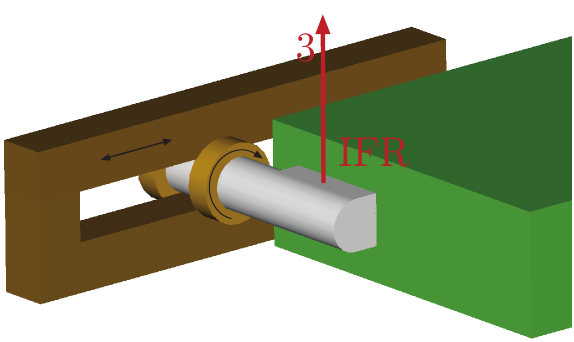}
} \vspace{-1.5ex}
\caption{Kinematics of the a) planar joint, b) hook joint, and c)
pin-in-slot joint.}
\label{fig2}
\end{figure}

\subsection{Planar Joint%
\label{secPlanarExamp}%
}

A 3 degree-of-freedom (DOF) 'planar joint' restricts the body to move on a
plane, i.e. it can perform translations in the plane and independent
rotations about an axis normal to the plane. Assigning the joint and RFR as
in figure \ref{fig2}a), the body can move in the 1-2-plane of the IFR. The
three geometric constraints are hence the third equation of (\ref{RevConstr1}%
) and the equations (\ref{RevConstr2},\ref{RevConstr3}). Accordingly, in
terms of body-fixed twists, the acceleration constraints are the third
equation of (\ref{accConstrTopSE3}) together with (\ref{RevConstr4},\ref%
{RevConstr5}). In terms of mixed velocities, these are the third equation in
(\ref{accConstrTopSO3xR3}) together with (\ref{RevConstr4},\ref{RevConstr5}).

The submanifold $H$, defined by the planar joint constraints, consists of
all configurations the body can attain when it can be freely located in the
plane. As explained in remark \ref{remarkPlanr} the update steps of both
formulations shall lead to exact constraint satisfaction. This is confirmed
by the simulation results. The constraints are satisfied regardless of the
integration step size (therefore not shown here).

\subsection{Hook Joint}

A hook joint allows for successive rotations about two (usually
perpendicular) rotation axes. With the joint assignment in figure \ref{fig2}%
b) the condition is that the 3-axis of the IFR and 2-axis of the body-fixed
RFR are perpendicular, and these two define the rotation axes. The
corresponding four geometric constraints are the three equations (\ref%
{RevConstr1}) and $R_{23}=0$. The acceleration constraints are omitted for
brevity. The EOM were integrated with initial velocity $\mathbf{\omega }%
_{0}^{\text{b}}=(0,\pi ,0.2\pi )$ rad/s.

The hook joint defines a 2-dimensional submanifold $H$ of the subgroup of
rotations about the IFR origin. There is no 2-dimensional subgroup of $%
SO\left( 3\right) $. Thus $H$ cannot be a subgroup of $SE\left( 3\right) $
nor of $SO\left( 3\right) \times \mathbb{R}^{3}$.

Figure \ref{figRotErrorHook} shows the expected dependence of the rotation
error $\varepsilon _{r}:=\left\vert R_{23}\right\vert $ on the step size.
This dependence is also observed for the position constraints when using the 
$SO\left( 3\right) \times \mathbb{R}^{3}$ formulation (figure \ref%
{figPosErrorHook}a)). The $SE\left( 3\right) $ update does, however, exactly
satisfy position constraint (figure \ref{figPosErrorHook}b)). This can be
explained by noting that points on the body are constrained to move on
spheres centered at the intersection of the joint axes. In $SE\left(
3\right) $ the completion group of the constrained rigid body velocities,
i.e. the smallest subgroup containing $H$, is $K=SO\left( 3\right) $. The
orbit $O_{x}=\{Cx|C\in K\}$ of a point $x\in E^{3}$ on the body by $K$ is a
sphere centered at the joint center. Hence, the position constraints are
respected.

\begin{figure}[h]
\vspace{-1ex} 
\centerline{a)~\includegraphics[width=8.3cm]{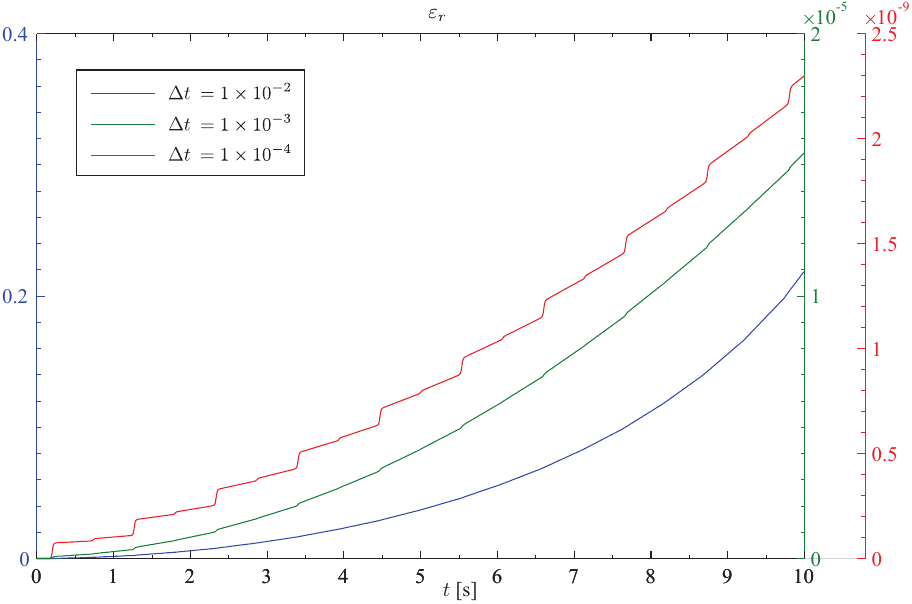}~~~~
b)~\includegraphics[width=8.3cm]{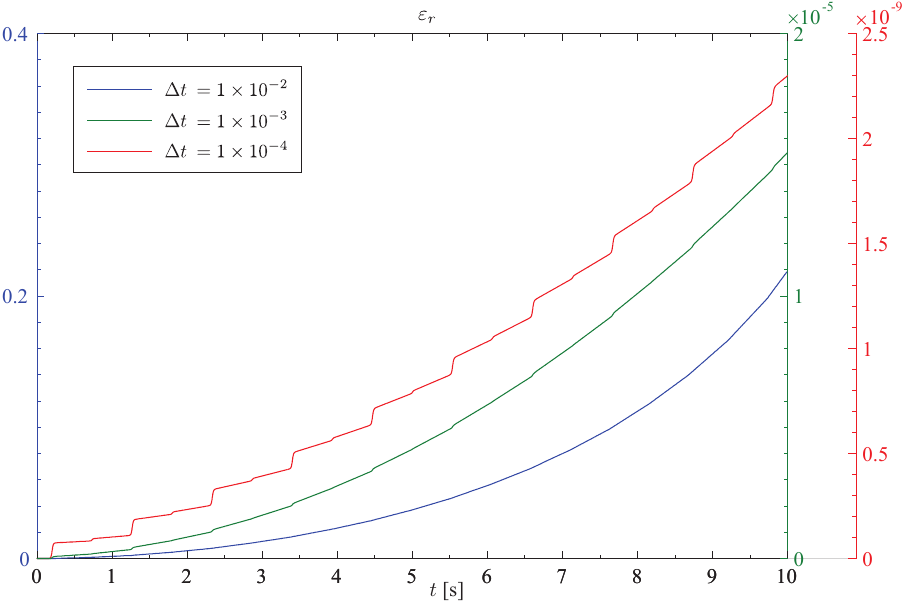}
} \vspace{-1.5ex}
\caption{Error in the hook joint rotation constraints for a) $SO(3)\times {%
\mathbb{R}^{3}}$, and b) $SE(3)$.}
\label{figRotErrorHook}
\end{figure}
\begin{figure}[h]
\vspace{-1ex} 
\centerline{a)~\includegraphics[width=8.5cm]{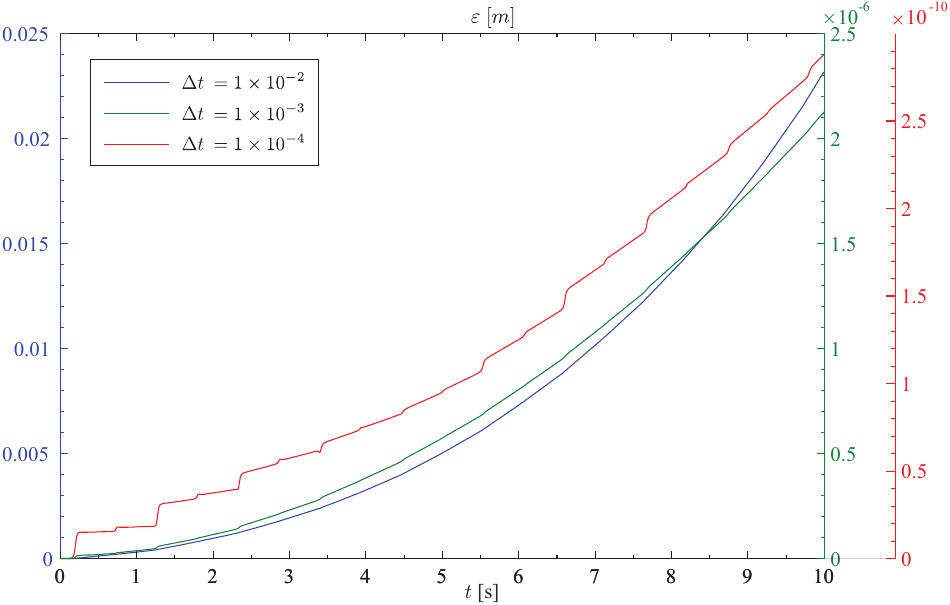}~~~~
b)~\includegraphics[width=8.3cm]{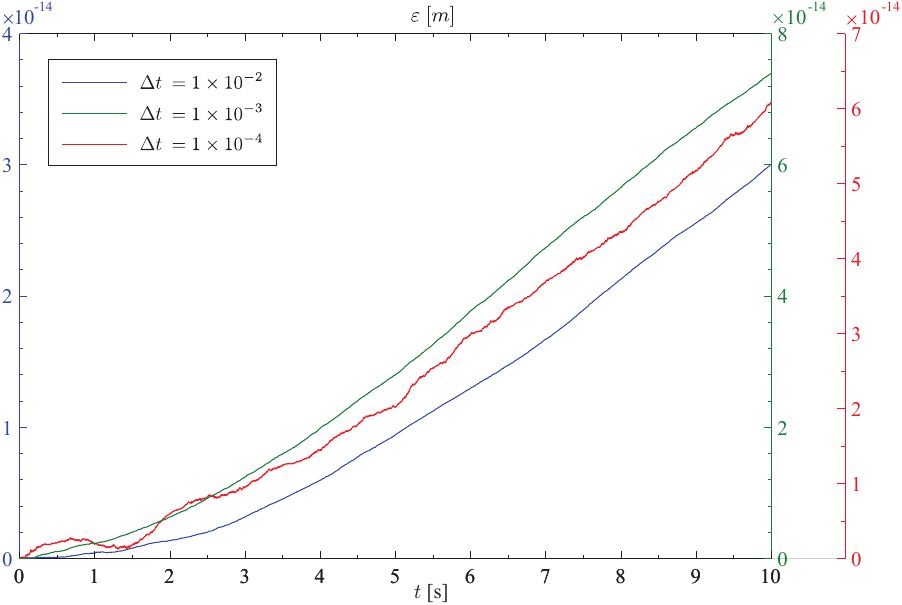}
}\vspace{-1.5ex} \vspace{4ex}
\caption{Error in the hook joint position constraints for a) $SO(3)\times {%
\mathbb{R}^{3}}$, and b) $SE(3)$.}
\label{figPosErrorHook}
\end{figure}

\subsection{Pin-in-Slot Joint}

The pin-in-slot joint is a 2 DOF higher kinematic pair. According to the RFR
in figure \ref{fig2}c) the four geometric constraints are the second and
third equation of (\ref{RevConstr1}) together with the equations (\ref%
{RevConstr2},\ref{RevConstr3}). The acceleration constraints for the
body-fixed twist formulation are the second and third equation of (\ref%
{accConstrTopSE3}) and the constraints (\ref{RevConstr4},\ref{RevConstr5}).
In mixed velocity formulation, the acceleration constraints are the second
and third equation of (\ref{accConstrTopSO3xR3}) together with the two
equations (\ref{RevConstr4},\ref{RevConstr5}).

The joint does not define a motion subgroup, so that the $SE\left( 3\right) $
update cannot respect the constraints. This also holds for $SO\left(
3\right) \times \mathbb{R}^{3}$ if a general RFR is used. This is confirmed
in figure \ref{figErrorPinInSlot} where numerical results are shown for the
initial velocities $\mathbf{\omega }_{0}^{\text{b}}=(0,\pi ,0.2\pi )$ rad/s
and $\mathbf{v}_{0}^{\text{b}}=\mathbf{p}\times \mathbf{\omega }_{0}^{\text{b%
}}+(1,0,0)~$m/s.

In order to clarify the significance of RFR for the $SO\left( 3\right)
\times \mathbb{R}^{3}$ formulation, consider the case where the translation
axis defined by the joint passes through the RFR at the COM, i.e. $\mathbf{p}%
=\mathbf{0}$. Then the motions in fact define the subgroup $SO_{0}%
\hspace{-0.6ex}%
\left( 2\right) \times \mathbb{R}$ of $SO\left( 3\right) \times \mathbb{R}%
^{3}$, and the constraints are exactly satisfied. The trivial numerical
results are omitted.

\begin{figure}[h]
\vspace{-1ex} 
\centerline{a)~\includegraphics[width=8.5cm]{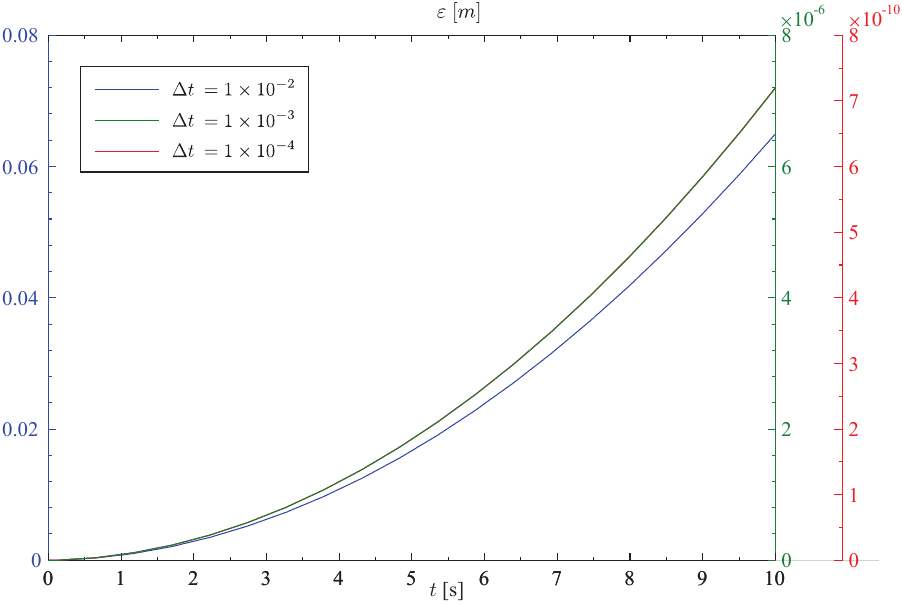}~~~~
b)~\includegraphics[width=8.3cm]{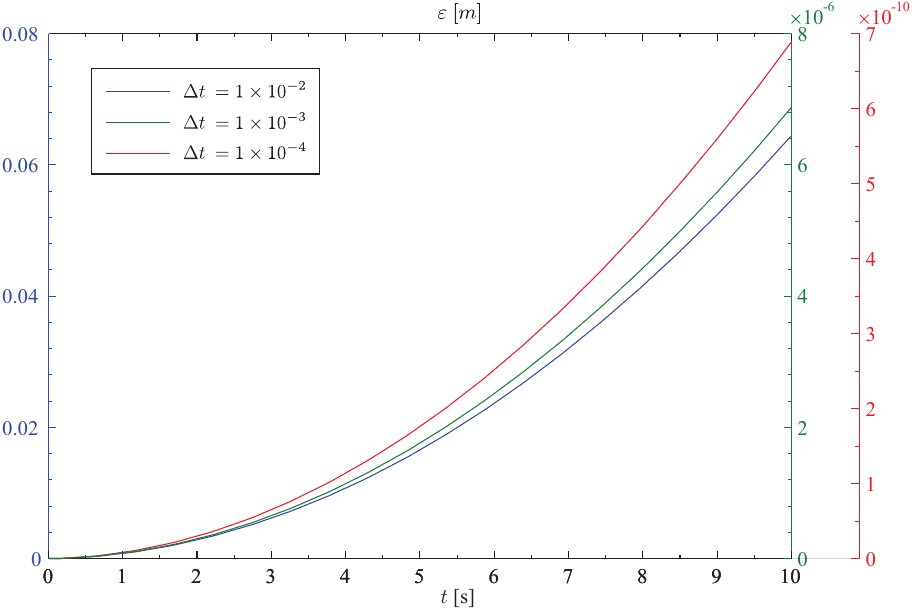}
} \vspace{-1.5ex}
\caption{Error in the pin-in-slot joint position constraints for a) $%
SO(3)\times {\mathbb{R}^{3}}$, and b) $SE(3)$.}
\label{figErrorPinInSlot}
\end{figure}

\section{Conclusion}

Constrained MBS comprising rigid bodies are frequently modeled by a system
of Newton-Euler equations subjected to the constraints due to geometric
interconnections. This is commonly treated as index 1 DAE system, and
numerically solved using ODE integration schemes. The classical MBS
formulation uses decoupled rotation and translation parameters. The
underlying geometric model is that of $SO\left( 3\right) \times \mathbb{R}%
^{3}$, although rigid body motions form the Lie group $SE\left( 3\right) $.
A major problem when integrating the index 1 formulation is the violation of
constraints. It has been observed, however, that the c-space Lie group
significantly affects the constraint satisfaction. This problem was here
addressed for the single rigid body constrained to an inertial system. For
the practically important case of lower pair joints, it is concluded that
the constraints are exactly satisfied (independently of the integration
scheme and accuracy) if $SE\left( 3\right) $ is used as c-space Lie group.
The practical implication of using $SE\left( 3\right) $ as c-space is the
use of screw coordinates $\mathbf{X}=\left( \mathbf{\xi },\mathbf{\eta }%
\right) $ as parameters, and coefficient matrices of the form (\ref%
{dexpInvSE3Park}) in the kinematic equations (\ref{ODEvec1b}) of the model.
Although already thoroughly investigated and published \cite%
{BorriBottasso1994,Bottasso1998,Bottasso2002,Liu1988}, the use of screw
coordinates for MBS modeling has not yet found due attention.

The presented result is valid for the single rigid body constrained to an
inertial frame. It cannot be immediately carried over to MBS. This shall be
the topic of further research.

\begin{acknowledgements}
The author acknowledges that this work has been partially supported by the Austrian COMET-K2 program of the Linz Center of Mechatronics (LCM).
\end{acknowledgements}

\end{document}